\documentclass[twoside]{article}
\usepackage[accepted]{aistats2018}
\usepackage{pdfpages}
\usepackage{amsmath,amsthm,amssymb,bm,epsfig,booktabs,multirow}
\usepackage[round]{natbib}
\usepackage[hidelinks]{hyperref}
\usepackage[title]{appendix}

\newtheorem{lemma}{Lemma}
\newtheorem{theorem}{Theorem}
\newtheorem{definition}{Definition}

\newcommand{\Le}{\left(}
\newcommand{\Ri}{\right)}
\newcommand{\N}{\mathbb{N}}

\author{} 

\runningauthor{Tomi Silander, Janne Lepp\"{a}-aho, Elias J\"{a}\"{a}saari, Teemu Roos}

\begin{document}

\twocolumn[

\aistatstitle{Quotient Normalized Maximum Likelihood Criterion for Learning Bayesian Network Structures}

\aistatsauthor{ Tomi Silander$^1$ \And Janne Lepp\"{a}-aho$^{2,4}$ \And Elias J\"{a}\"{a}saari$^{2,3}$ \And Teemu Roos$^{2,4}$   }

\aistatsaddress{ $^1$NAVER LABS Europe \And $^2$Helsinki Institute for \\ Information Technology HIIT \And $^3$Department of CS,\\Aalto University \And $^4$Department of CS,\\University of Helsinki} ]


\begin{abstract}
We introduce an information theoretic criterion for Bayesian network
structure learning which we call quotient normalized maximum
likelihood (qNML). In contrast to the closely related factorized
normalized maximum likelihood criterion, qNML satisfies the property
of score equivalence. It is also decomposable and completely free
of adjustable hyperparameters. For practical computations, we identify
a remarkably accurate approximation proposed earlier by Szpankowski
and Weinberger. Experiments on both simulated and real data
demonstrate that the new criterion leads to parsimonious models with
good predictive accuracy.
\end{abstract}

\section{INTRODUCTION}
\label{sec:intro}
Bayesian networks~\cite{Pear88} are popular models for presenting
multivariate statistical dependencies that may have been induced by
underlying causal mechanisms.  Techniques for learning the structure
of Bayesian networks from observational data have therefore been used
for many tasks such as discovering cell signaling pathways from
protein activity data~\cite{bn4sigpath02}, revealing the business
process structures from transaction logs ~\cite{bn4bpmining} and
modeling brain-region connectivity using fMRI
data~\cite{bn4brainconnect}.

Learning the structure of statistical dependencies can be seen as a
model selection task where each model is a different hypothesis about
the conditional dependencies between sets of variables. Traditional
model selection criteria such as the Akaike information criterion
(AIC)~\cite{Akai73} and the Bayesian information criterion
(BIC)~\cite{Schw78} have also been used for the task, but recent
comparisons have not been favorable for AIC, and BIC appears to
require large sample sizes in order to identify appropriate
structures~\cite{cosco.pgm08a,Liu2012}. Traditionally, the most
popular criterion has been the Bayesian marginal
likelihood~\cite{Heck95b} and its BDeu variant (see
Section~\ref{sec:bns}), but studies~\cite{cosco.uai07,Steck08} show
this criterion to be sensitive to hyperparameters and to
yield undesirably complex models for small sample sizes.

The information-theoretic normalized maximum likelihood (NML)
criterion~\cite{Shta87,Riss96a} would otherwise be a potential
candidate for a good criterion, but its exact calculation is likely to
be prohibitively expensive. In 2008, Silander et al. introduced a
hyperparameter-free, NML inspired criterion called the factorized NML
(fNML)~\cite{cosco.pgm08a} that was shown to yield good predictive
models without such sensitivity problems. However, from the structure
learning point of view, fNML still sometimes appears to yield
overly complex models. In this paper we introduce another NML related
criterion, the \textit{quotient NML} (qNML) that yields simpler models
without sacrificing predictive accuracy. Furthermore, unlike fNML,
qNML is \textit{score equivalent}, i.e., it gives equal scores to
structures that encode the same independence and dependence
statements. Like other common model selection criteria, qNML is also
consistent.

We next briefly introduce Bayesian networks and review the BDeu and
fNML criteria and then introduce the qNML criterion.  We also
summarize the results for 20 data sets to back up our claim that qNML
yields parsimonious models with good predictive capabilities. The
experiments with artificial data generated from real-world Bayesian networks demonstrate the capability of our score to quickly learn a structure close to the generating one.

\section{BAYESIAN NETWORKS}
\label{sec:bns}

Bayesian networks are a general way to describe the dependencies
between the components of an $n$\nobreakdash-dimensional random data
vector. In this paper we only address the case in which the component
$X_{i}$ of the data vector $X=(X_{1},\ldots,X_{n})$ may take any of
the discrete values in a set $\{1,\ldots,r_{i}\}$.  Despite denoting
the values with small integers, the model will treat each $X_i$
as a categorical variable.

\subsection{Likelihood}
\label{ssec:likelihood}

A Bayesian network $B=(G,\theta)$ defines a probability distribution for
$X$. The component $G$ defines the structure of the model as a
directed acyclic graph (DAG) that has exactly one node for each component of
$X$. The structure $G=(G_{1},\ldots,G_{n})$ defines for each
variable/node $X_{i}$ its (possibly empty) parent set $G_{i}$, i.e.,
the nodes from which there is a directed edge to the variable
$X_{i}$.

Given a realization $x$ of $X$, we denote the sub\nobreakdash-vector
of $x$ that consists of the values of the parents of $X_{i}$ in $x$ by
$G_{i}(x)$. It is customary to enumerate all the possible
sub\nobreakdash-vectors $G_{i}(x)$ from $1$ to $q_{i}=\prod_{h\in
  G_{i}}r_{h}.$ In case $G_{i}$ is empty, we define $q_{i}=1$ and
$P(G_{i}(x)=1)=1$ for all vectors $x$.

For each variable $X_{i}$ there is a $q_{i}\times r_{i}$ table
$\theta_{i}$ of parameters whose $k^{\textnormal{th}}$ column
on the $j^{\textnormal{th}}$ row $\theta_{ij}$ defines the conditional
probability $P(X_{i}=k\mid G_{i}(X)=j;\theta)=\theta_{ijk}$.  With
structure $G$ and parameters $\theta$, we can now express the
likelihood function of the model as
\begin{equation}
P(x|G,\theta)=\prod_{i=1}^{n}P(x_{i}\mid
G_{i}(x);\theta_{i})=\prod_{i=1}^{n}\theta_{iG_{i}(x)x_{i}}.
\end{equation}

\subsection{Bayesian Structure Learning}

Score-based Bayesian learning of Bayesian network structures evaluates the goodness of different structures $G$ using their
posterior probability $P(G|D,\alpha)$, where $\alpha$ denotes the
hyperparameters for the model parameters $\theta$, and $D$ is a
collection of $N$ $n$\nobreakdash-dimensional i.i.d. data vectors
collected to a $N\times n$ design matrix. We use the notation $D_i$ to
denote the $i^\text{th}$ column of the data matrix and the notation $D_V$ to denote
the columns that correspond to the variable subset $V$. We also write
$D_{i,G_i}$ for $D_{\{i\}\cup G_i}$ and denote the entries of the
column $i$ on the rows on which the parents $G_i$ contain the value
configuration number $j$ by $D_{i,G_i=j}$, $j\in\{1,\ldots,q_i\}$.

It is common to assume the uniform prior for structures, in which case
the objective function for structure learning is reduced to the
marginal likelihood $P(D|G,\alpha)$.  If the model parameters
$\theta_{ij}$ are further assumed to be independently Dirichlet
distributed only depending on $i$ and $G_{i}$ and the data $D$ is
assumed to have no missing values, the marginal likelihood can be
decomposed as
\begin{eqnarray}
\label{eqn:bayesmix}
\lefteqn{P(D|G,\alpha)}\nonumber\\
&&=\prod_{i=1}^{n}\prod_{j=1}^{q_i}P(D_{i,G_i=j};\alpha)\nonumber\\
&&=\prod_{i=1}^{n}\prod_{j=1}^{q_i}\int P(D_{i,G_i=j}|\theta_{ij})P(\theta_{ij};\alpha) d\theta_{ij}.
\end{eqnarray}
In coding terms this means that each data column $D_i$ is first
partitioned based on the values in columns $G_i$, and each part is
then coded using a Bayesian mixture that can be expressed in closed
form~\cite{Bunt91, Heck95}. 

\begin{figure}
\centering
\includegraphics[width=8cm,height=5cm]{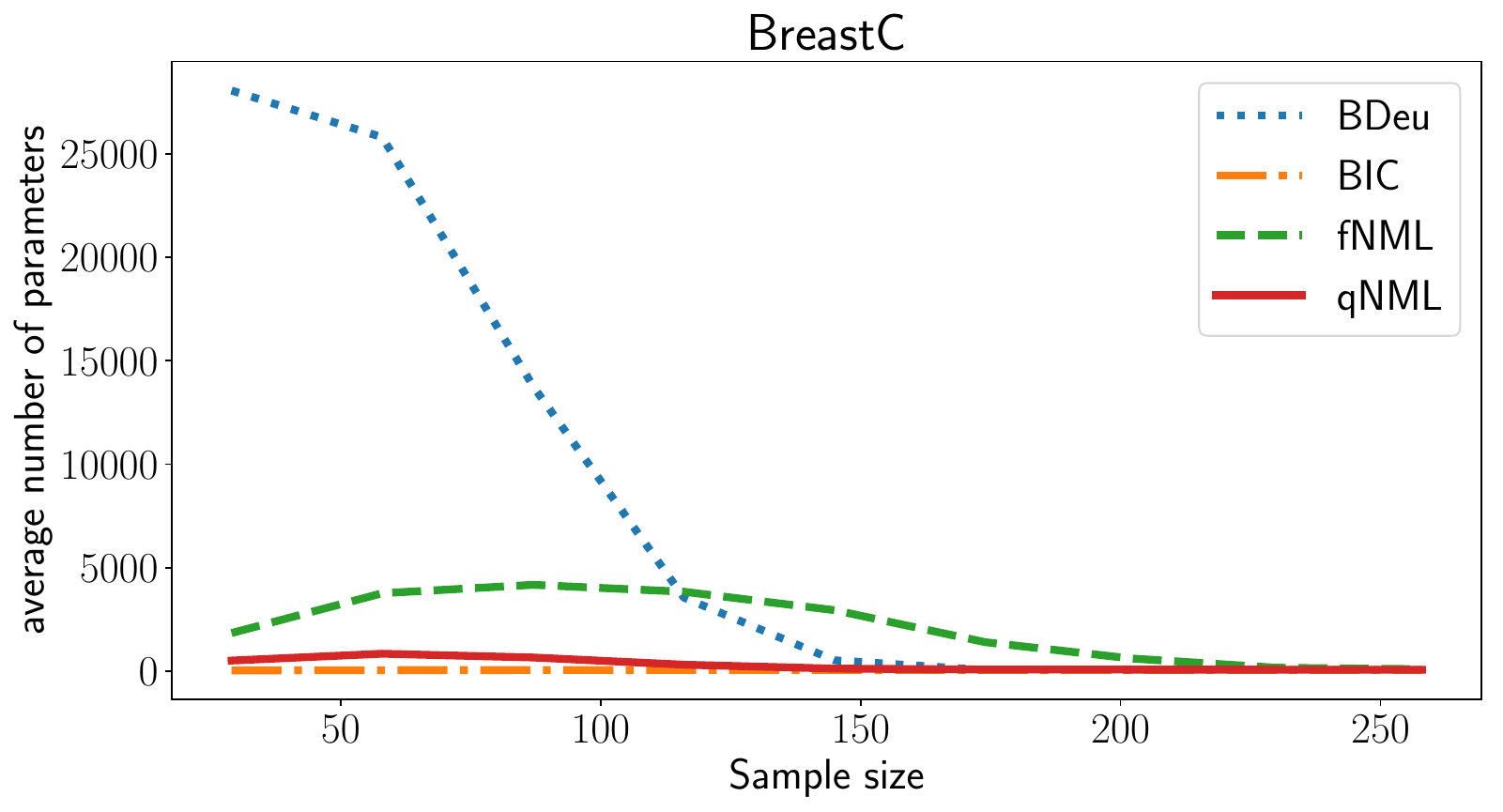}
\caption{Number of parameters in a breast cancer model as a function
  of sample size for different model selection criteria.}
\label{fig:bcnpmean}
\end{figure}

\subsection {Problems, Solutions and Problems}

Finding satisfactory Dirichlet hyperparameters for the Bayesian
mixture above has, however, turned out to be problematic. Early on,
one of the desiderata for a good model selection criterion was that it
is score equivalent, i.e., it would yield equal scores for
essentially equivalent models~\cite{Verm90}.  For example, the score
for the structure $X_1\rightarrow X_2$ should be the same as the score
for the model $X_2 \rightarrow X_1$ since they both correspond to the
hypothesis that variables $X_1$ and $X_2$ are statistically dependent
on each other.  It can be shown~\cite{Heck95} that to achieve this,
not all the hyperparameters $\alpha$ are possible and for practical
reasons Buntine~\cite{Bunt91} suggested a so-called BDeu score with
just one hyperparameter $\alpha\in R_{++}$ so that
$\theta_{ij\cdot}\sim Dir(\frac{\alpha}{q_i
  r_i},\ldots,\frac{\alpha}{q_i r_i})$.  However, it soon turned out
that the BDeu score was very sensitive to the selection of this
hyperparameter~\cite{cosco.uai07} and that for small sample sizes this
method detects spurious correlations~\cite{Steck08} leading to models
with suspiciously many parameters.

Recently, Suzuki~\cite{Suzuki2017} discussed the theoretical
properties of the BDeu score and showed that in certain settings BDeu
has the tendency to add more and more parent variables for a child node
even though the empirical conditional entropy of the child given the
parents has already reached zero. In more detail, assume that in our
data $D$ the values of $X_i$ are completely determined by variables
in a set $Z$, so that the empirical entropy $H_N(X_i | Z)$ is zero. Now, if we can further
find one or more variables, denoted by $Y$, whose values are
determined completely by the variables in $Z$, then BDeu will prefer
the set $Z\cup Y$ over $Z$ alone as the parents of $X_i$. Suzuki
argues that this kind of behavior violates \textit{regularity} in model
selection as the more complex model is preferred over a simpler one
even though it does not fit the data any better. The phenomenon seems
to stem from the way the hyperparameters for the Dirichlet
distribution are chosen in BDeu as using Jeffreys' prior,
$\theta_{ijk}\sim Dir(\frac{1}{2},\ldots,\frac{1}{2})$, does not suffer
from this anomaly. However, using Jeffreys' prior causes marginal
likelihood score not to be score equivalent. In Section \ref{sec:regularity}, we will give the formal definition of regularity and state that qNML is regular. In addition, we provide a proof of regularity for fNML criterion, which has not appeared in the literature before. The detailed proofs can be found in Appendix B in the Supplementary Material.     

A natural solution to avoid parameter sensitivity of BDeu would be to
use a normalized maximum likelihood (NML)
criterion~\cite{Shta87,Riss96a}, i.e., to find the structure $G$ that
maximizes
\begin{equation}
P_{NML}(D;G)=\frac{P(D|\hat\theta(D;G))}{\sum_{D'}{P(D'|\hat\theta(D';G))}},
\end{equation}
where $\hat\theta$ denotes the (easy to find) maximum likelihood
parameters and the sum in the denominator goes over all the possible
$N\times n$ data matrices. This information-theoretic NML criterion
can be justified from the minimum description length point of view
\cite{Riss78,Grun07}. It has been shown to be robust with respect to
different data generating mechanisms where a good choice of prior
is challenging, see~\cite{eggeling2014robust,maatta16}. While it is
easy to see that the NML criterion satisfies the requirement of giving
equal scores to equal structures, the normalizing constant renders the
computation infeasible.

Consequently, Silander et al.~\cite{cosco.pgm08a}
suggested solving the BDeu parameter sensitivity problem by using the
NML code for the column partitions, i.e., changing the Bayesian mixture
in equation~(\ref{eqn:bayesmix}) to
\begin{equation}
P^1_{NML}(D_{i,G_i=j};G)=\frac{P(D|\hat\theta(D_{i,G_i=j};G))}{\sum_{D'}{P(D'|\hat\theta(D';G))}},
\end{equation}
where $D'\in{\{1,\ldots,r_i\}}^{|D_{i,G_i=j}|}$.  The logarithm of the
denominator is often called the regret, since it indicates the extra
code length needed compared to the code length obtained using the (a
priori unknown) maximum likelihood parameters. The regret for
$P^1_{NML}$ depends only on the length $N$ of the categorical data
vector with $r$ different categorical values,
\begin{equation}
reg(N,r)=\log \sum_{D\in \{1,\ldots,r\}^N} P(D|\hat\theta(D)).
\end{equation}
While the naive
computation of the regret is still prohibitive, Silander et al.\
approximate it efficiently using a so-called Szpankowski
approximation~\cite{cosco.aistat03}:
\begin{eqnarray}
\label{eqn:szp1}
\lefteqn{reg(N,r) \approx \frac{\sqrt{2} r \Gamma{\left(\frac{r}{2} \right)}}
                               {3 \sqrt{N} \Gamma{\left(\frac{r-1}{2}  \right)}}} \\
&&+ \left(\frac{r-1}{2}\right) \log{\left (\frac{N}{2} \right )}
- \log \Gamma{\left(\frac{r}{2} \right)} + \frac{1}{2} \log{\left (\pi \right )}\nonumber\\
&&- \frac{r^{2} \Gamma^{2}{\left(\frac{r }{2} \right)}}
         {9N \Gamma^{2}{\left(\frac{r-1}{2}\right)}}
+ \frac{2r^3-3r^2-2r+3}{36N}\nonumber.
\end{eqnarray}

However, equation (\ref{eqn:szp1}) is derived only for the case
where $r$ is constant and $N$ grows. While with
fNML it is typical that $N$ is large compared to $r$, an
approximation for all ranges of $N$ and $r$ derived by
Szpankowski and Weinberger~\cite{Szpankowski2012} can also be used:
\begin{eqnarray}
\label{eqn:szp2}
    reg(N, r) & \approx & N\left(\log{\alpha} + (\alpha + 2) \log{C_\alpha}
                - \frac{1}{C_\alpha}\right)\nonumber \\
    && - \frac{1}{2} \log{\left(C_\alpha + \frac{2}{\alpha}\right)},
\end{eqnarray}
where $\alpha = \frac{r}{N}$ and
    $C_\alpha = \frac{1}{2} + \frac{1}{2} \sqrt{1 + \frac{4}{\alpha}}$.
These approximations are compared in Table \ref{tbl:regrets} to the
exact regret for various values of $N$ and $r$.  For a constant $N$,
equation (\ref{eqn:szp1}) provides a progressively worse approximation
as $r$ grows. Equation (\ref{eqn:szp2}) on the other hand is a good
approximation of the regret regardless of the ratio of $N$ and $r$.
In our experiments, we will use this approximation for 
implementation of the qNML criterion.

fNML solves the parameter sensitivity problem and yields predictive
models superior to BDeu.  However, the criterion does not satisfy the
property of giving the same score for models that correspond to the
same dependence statements. The score equivalence is usually viewed desirable when DAGs are considered only as models for conditional independence, without any causal interpretation. Furthermore, the learned structures are
often rather complex (see Figure~\ref{fig:bcnpmean}) which also
hampers their interpretation. The quest for a model selection
criterion that would yield more parsimonious, easier to interpret, but
still predictive Bayesian networks structures is one of the main
motivations for this work.

\begin{table}
\caption{Regret values for various values of $N$ and $r$.}
\label{tbl:regrets}
\begin{center}
\begin{tabular}{crrrr}
N & r & eq. (\ref{eqn:szp1}) & eq. (\ref{eqn:szp2}) & exact \\
\midrule
\multirow{4}{*}{50} & 10 & 13.24 & 13.26 & 13.24 \\
& 100 & 62.00 & 60.01 & 60.00 \\
& 1000 & 491.63 & 153.28 & 153.28 \\
& 10000 & 25635.15 & 265.28 & 265.28 \\
\midrule
\multirow{4}{*}{500} & 10 & 22.67 & 22.69 & 22.67 \\
& 100 & 144.10 & 144.03 & 144.03 \\
& 1000 & 624.35 & 603.93 & 603.93 \\
& 10000 & 4927.24 & 1533.38 & 1533.38 \\
\midrule
\multirow{4}{*}{5000} & 10 & 32.74 & 32.76 & 32.74 \\
& 100 & 247.97 & 247.97 & 247.97 \\
& 1000 & 1452.51 & 1451.78 & 1451.78 \\
& 10000 & 6247.83 & 6043.16 & 6043.16 \\
\bottomrule
\end{tabular}
\end{center}
\end{table}
 
\section{QUOTIENT NML SCORE}

We will now introduce a quotient normalized maximum likelihood (qNML)
criterion for learning Bayesian network structures.  While equally
efficient to compute as BDeu and fNML, it is free from
hyperparameters, and it can be proven to give equal scores to
equivalent models. Furthermore, it coincides with the actual NML score
for exponentially many models. In our empirical tests it produces
models featuring good predictive performance with significantly
simpler structures than BDeu and fNML.

Like BDeu and fNML, qNML can be expressed as a product of $n$ terms,
one for each variable, but unlike the other two, it is not based on
further partitioning the corresponding data column
\begin{eqnarray}
\label{eqn:qnmldef}
s^{qNML}(D;G) & := & \sum_{i=1}^n s^{qNML}_i(D;G)\\
& := & \sum_{i=1}^n \log \frac{P^1_{NML}(D_{i,G_i};G)}
                             {P^1_{NML}(D_{G_i};G)}.\nonumber
\end{eqnarray}
The trick here is to model a subset of columns as though there were no
conditional independencies among the corresponding variables $S
\subset X$.  In this case, we can collapse the $\prod_{X_i\in S} r_i$
value configurations and consider them as values of a single variable
with $\prod_{X_i\in S} r_i$ different values which can then be modeled
with a one-dimensional $P^1_{NML}$ code.  The $s^{qNML}$ score does
not necessarily define a distribution for $D$, but it is easy to
verify that it coincides with the NML score for all networks
that are composed of fully connected components.  The number of such
networks is lower bounded by the number of nonempty partitions of a
set of $n$ elements, i.e., the $n^\text{th}$ Bell number.

We are now ready to prove some important properties of the qNML score.

\subsection {qNML Is Score Equivalent}

qNML yields equal scores for network structures that encode the same set
of independencies. Verma and Pearl~\cite{Verm90} showed that the
equivalent networks are exactly those which a) are the same when directed
arcs are substituted by undirected ones and b) which have the same
\textit{V-structures}, i.e. the variable triplets $(A,B,C)$ where both
$A$ and $B$ are parents of $C$, but there is no arc between $A$ and
$B$ (in either direction).  Later, Chickering~\cite{Chick95} showed
that all the equivalent network structures, and only those structures,
can be reached from each other by reversing, one by one, the so-called
\textit{covered arcs}, i.e. the arcs from node $A$ to $B$, for which
$B$'s parents other than $A$ are exactly  $A$'s parents
($G_B=\{A\}\cup G_A$).

We will next state this as a
theorem and sketch a proof for it. A more detailed proof appears in Appendix A in the Supplementary Material.
\begin{theorem}
  \label{thm:scoreqv}
  Let $G$ and $G'$ be two Bayesian network structures that differ only
  by a single covered arc reversal, i.e., the arc from $A$ to $B$ in $G$
  has been reversed in $G'$ to point from $B$ to $A$, then
  $$s^{qNML}(D;G)=s^{qNML}(D;G').$$
\end{theorem}
\begin{proof}
  Now the scores for structures can be decomposed as
  $s^{qNML}(D;G)=\sum_{i=1}^{n}s_i^{qNML}(D;G)$ and
  $s^{qNML}(D;G')=\sum_{i=1}^{n}s_i^{qNML}(D;G')$.  Since only the
  terms corresponding to the variables $A$ and $B$ in these sums are
  different, it is enough to show that the sum of these two terms are
  equal for $G$ and $G'$. Since we can assume the data to be fixed we
  lighten up the notation and write
  $P^1_{NML}(i,G_i) := P^1_{NML}(D_{i,G_i};G)$ and
  $P^1_{NML}(G_i)   := P^1_{NML}(D_{G_i};G)$. Now
  \begin{eqnarray}
    \lefteqn{s_A^{qNML}(D;G)+s_B^{qNML}(D;G)} \nonumber\\
    && =\log\frac{P^1_{NML}(A,G_{A})}{P^1_{NML}(G_{A})}
            \frac{P^1_{NML}(B,G_{B})}{P^1_{NML}(G_{B})}\nonumber\\
    && =\log 1\cdot\frac{P^1_{NML}(B,G_{B})}{P^1_{NML}(G_{A})}\nonumber\\
    && =\log \frac{P^1_{NML}(B,G'_{B})}{P^1_{NML}(G'_{A})}
             \frac{P^1_{NML}(A,G'_{A})}{P^1_{NML}(G'_{B})}\nonumber\\
 && =s_A^{qNML}(D;G')+s_B^{qNML}(D;G'),\nonumber
\end{eqnarray}
  using the equations $\{A\}\cup G_A = G_B$, $\{B\}\cup G'_B = G'_A$,
  $\{B\}\cup G_B = \{A\} \cup G'_A$, and $G_A = G'_B$ which follow
  easily from the definition of covered arcs.
\end{proof}

\subsection{qNML is Consistent}

One important property possessed by nearly every model selection
criterion is consistency. In our context, consistency means that given
a data matrix with $N$ samples coming from a distribution faithful to
some DAG $G$, the qNML will give the highest score to the true graph $G$
with a probability tending to one as $N$ increases. We will show this
by first proving that qNML is asymptotically equivalent to the widely used
BIC criterion which is known to be consistent \cite{Schw78, Haug88}.
The outline of this proof follows a similar pattern to that in
\cite{SilanderIJAR10} where the consistency of fNML was proved.

The BIC criterion can be written as
\begin{equation}\label{BIC}
\textnormal{BIC}(D;G) = \sum_{i = 1}^n \log P(D_i \ | \ \hat{\theta}_{i | G_i} ) - \frac{q_i(r_i - 1)}{2} \log N,
\end{equation}
where $\hat{\theta}_{i | G_i}$ denotes the maximum likelihood parameters of
the conditional distribution of variable $i$ given its parents in
$G$. 

Since both the BIC and qNML scores are decomposable, we can focus on
studying the local scores. We will next show that, asymptotically, the
local qNML score equals the local BIC score. This is formulated in the
following theorem:

\begin{theorem}\label{consistency}
Let $r_i$ and $q_i$ denote the number of possible values for variable
$X_i$ and its possible configurations of parents $G_i$,
respectively. As $N \to \infty$,
$$
s^{qNML}_i(D;G) =  \log P(D_i \ | \ \hat{\theta}_{i | G_i} )  - \frac{q_i(r_i - 1)}{2} \log N.
$$
\end{theorem}

In order to prove this, we start with the definition of qNML and write
\begin{align}\label{qnmlDef2}
s^{qNML}_i(D;G) &= \log \frac{P(D_{i, G_i} \ | \ \hat{\theta}_{i, G_i}
  )}{P(D_{G_i} \ | \ \hat{\theta}_{G_i} )} \notag \\ & -(reg(N,q_i
r_i) - reg(N,q_i)).
\end{align}

By comparing the equations (\ref{BIC}) and (\ref{qnmlDef2}), we see
that proving our clam boils down to showing two things: 1) the terms
involving the maximized likelihoods are equal and 2) the penalty terms
are asymptotically equivalent. We will formulate these as two
lemmas.

\begin{lemma}\label{MLLemma} The maximized likelihood terms in equations (\ref{BIC}) and (\ref{qnmlDef2}) are equal:    
$$
\frac{P(D_{i, G_i} \ | \ \hat{\theta}_{i, G_i} )}{P(D_{G_i} \ | \ \hat{\theta}_{G_i} )} = P(D_i \ | \ \hat{\theta}_{i | G_i} ).
$$
\end{lemma}

\begin{proof}
We can write the terms on the left side of the equation as
\begin{eqnarray*}
P(D_{i, G_i} \ | \ \hat{\theta}_{i, G_i}) &=& \prod_{j,k} \Le \frac{N_{ijk}}{N}  \Ri^{N_{ijk}}, \text{ and }\\
P(D_{G_i} \ | \ \hat{\theta}_{G_i} ) &=&  \prod_{j} \Le \frac{N_{ij}}{N}  \Ri^{N_{ij}}.
\end{eqnarray*}
Here, $N_{ijk}$ denotes the number of times we observe $X_i$ taking value $k$ when its parents are in $j^\text{th}$ configuration in our data matrix $D$. Also, $N_{ij} = \sum_k N_{ijk}$ (and $\sum_{k,j}N_{ijk} = N$ for all $i$).
Therefore,
\begin{align*}
\frac{P(D_{i, G_i} \ | \ \hat{\theta}_{i, G_i} )}{P(D_{G_i} \ | \ \hat{\theta}_{G_i} )} &= \frac{ \prod_{j,k} \Le \frac{N_{ijk}}{N}  \Ri^{N_{ijk}}}{\prod_{j} \Le \frac{N_{ij}}{N}  \Ri^{N_{ij}}} \\
&= \frac{ \prod_{j,k} \Le \frac{N_{ijk}}{N}  \Ri^{N_{ijk}}}{\prod_{j}\prod_{k} \Le \frac{N_{ij}}{N}  \Ri^{N_{ijk}}}\\
&= P(D_i \ | \ \hat{\theta}_{i | G_i} ).
\end{align*} 

\end{proof}
Next, we consider the difference of regrets in
(\ref{qnmlDef2}) which corresponds to the penalty term of BIC. The
following lemma states that these two are asymptotically equal:

\begin{lemma}\label{penaltyLemma}
As $N \to \infty$,
$$reg(N,q_i r_i) - reg(N,q_i) = \frac{q_i(r_i - 1)}{2}\log N + O(1).$$   
\end{lemma}
\begin{proof}
The regret
for a single multinomial variable with $m$ categories can be written
asymptotically as
\begin{equation}\label{regretAsymp}
reg(N,m) = \frac{m-1}{2}\log N + O(1).
\end{equation}
For the more precise statement with the underlying assumptions (which are fulfilled in the multinomial case) and for the proof, we refer to \cite{Riss96a, Grun07}. Using this, we have
\begin{align*}
reg(N,q_i r_i)& -  reg(N,q_i) \\ 
&= \frac{q_ir_i-1}{2}\log N-\frac{q_i-1}{2}\log N + O(1) \\
&= \frac{q_ir_i-1-q_i + 1}{2}\log N + O(1) \\
&= \frac{q_i(r_i - 1)}{2} \log N + O(1).
\end{align*}
\end{proof}
This concludes our proof since Lemmas \ref{MLLemma} and
\ref{penaltyLemma} imply Theorem \ref{consistency}.

\subsection{qNML Equals NML for Many Models}
The fNML criterion can be seen as a computationally feasible
approximation of the more desirable NML criterion.  However, the fNML
criterion equals the NML criterion only for the Bayesian network
structure with no arcs.  It can be shown that the qNML criterion
equals the NML criterion for all the networks $G$ whose connected
components are tournaments (i.e., complete directed acyclic subgraphs of
$G$). These networks include the empty network, the fully connected
one and many networks in between having different complexity. While
the generating network is unlikely to be composed of tournament
components, the result increases the plausibility that qNML is a
reasonable approximation for NML in general\footnote{A claim that is
  naturally subject for further study.}.

\begin{theorem}
If $G$ consists of $C$ connected components $(G^1,\ldots,G^C)$ with
variable sets $(V^1,\ldots,V^C)$, then $\log P_{NML}(D;G) = s^{qNML}(D;G)$
for all data sets $D$.
\end{theorem}
\begin{proof}
The full proof can be found in Appendix C in the Supplementary
Material.  The proof first shows that NML decomposes for these
particular structures, so it is enough to show the equivalence for
fully connected graphs.
It further derives the number $a(n)$ of
different $n$-node networks whose connected components are
tournaments, which turns out to be the formula for OEIS sequence
A000262\footnote{https://oeis.org/A000262}.
In general this sequence grows rapidly; $1, 1, 3, 13, 73, 501, 4051,
37633, 394353, 4596553, \ldots$.
\end{proof}

\subsection{qNML is Regular}\label{sec:regularity}

Suzuki \cite{Suzuki2017} defines regularity for a scoring function $Q_n(X \mid Y)$ as follows:
\begin{definition}
Assume $H_N(X \mid Y') \leq H_N(X \mid Y)$, where $Y' \subset Y.$ We say that $Q_N(\cdot \mid \cdot)$ is regular if $Q_N(X \mid Y') \geq Q_N(X \mid Y)$.
\end{definition}
In the definition, $N$ denotes the sample size, $X$ is some random variable, $Y$ denotes the proposed parent set for $X$, and $H_N(\cdot \mid \cdot)$ refers to the empirical conditional entropy. Suzuki \cite{Suzuki2017} shows that BDeu violates this principle and demonstrates that this can cause the score to prefer more complex networks even though the data do not support this. Regular scores are also argued to be computationally more efficient when applied with branch-and-bound type algorithms for Bayesian network structure learning \cite{Suzuki2017_2}. 

By analyzing the penalty term of the qNML scoring function, one can prove the following statement:
\begin{theorem}
qNML score is regular.
\end{theorem}
\begin{proof}
The proof is given in Appendix B in the Supplementary Material.
\end{proof}
As fNML criterion differs from qNML only by how the penalty term is defined, we obtain the following result with little extra work:
\begin{theorem}
fNML score is regular.
\end{theorem}
\begin{proof}
The proof is given in Appendix B in the Supplementary Material.
\end{proof}

Suzuki \cite{Suzuki2017} independently introduces a Bayesian Dirichlet
quotient (BDq) score that can also be shown to be score equivalent
and regular.  However, like BDeu, this score features a single
hyperparameter $\alpha$, and our initial studies suggest that BDq is
also very sensitive to this hyperparameter (see Appendix D in the
Supplementary Material), the issue that was one of the main
motivations to develop a parameter-free model selection criterion like
qNML.

\section{EXPERIMENTAL RESULTS}

We empirically compare the capacity of qNML to that of BIC, BDeu ($\alpha = 1$) and
fNML in identifying the data generating structures, and producing
models that are predictive and parsimonious.  It seems that none of
the criteria uniformly outperform the others in all these desirable
aspects of model selection criteria.

\subsection{Finding Generating Structure}

In our first experiment, we took five widely used benchmark Bayesian networks\footnote{Bayesian Network Repository: \url{http://www.bnlearn.com/bnrepository/}}, sampled data from them, and tried to learn the generating structure with the different scoring functions using various sample sizes. We used the following networks: Asia ($n=5$, 8 arcs), Cancer ($n=5$, 4 arcs), Earthquake ($n=5$, 4 arcs), Sachs ($n=11$, 17 arcs) and Survey ($n=6$, 6 arcs). These networks were picked in order to use the dynamic programming based exact structure learning~\cite{cosco.uai06} which limited the number $n$ of variables
to less than 20. We measured the quality of the learned structures using structural Hamming distance (SHD) \cite{Tsamardinos2006}.

Figure \ref{fig:all_shd} shows SHDs for all the scoring criteria for each network. Sample sizes range from $10$ to $10000$ and the shown results are averages computed from $1000$ repetitions. None of the scores dominates in all settings considered. BIC fares well when the sample size is small as it tends to produce a nearly empty graph which is a good answer in terms of SHD when the generating networks are relatively sparse. qNML obtains strong results in the Earthquake and Asia networks, being the best or the second best with all the sample sizes considered.

\begin{figure}[h]
\centering
\includegraphics[width=\columnwidth]{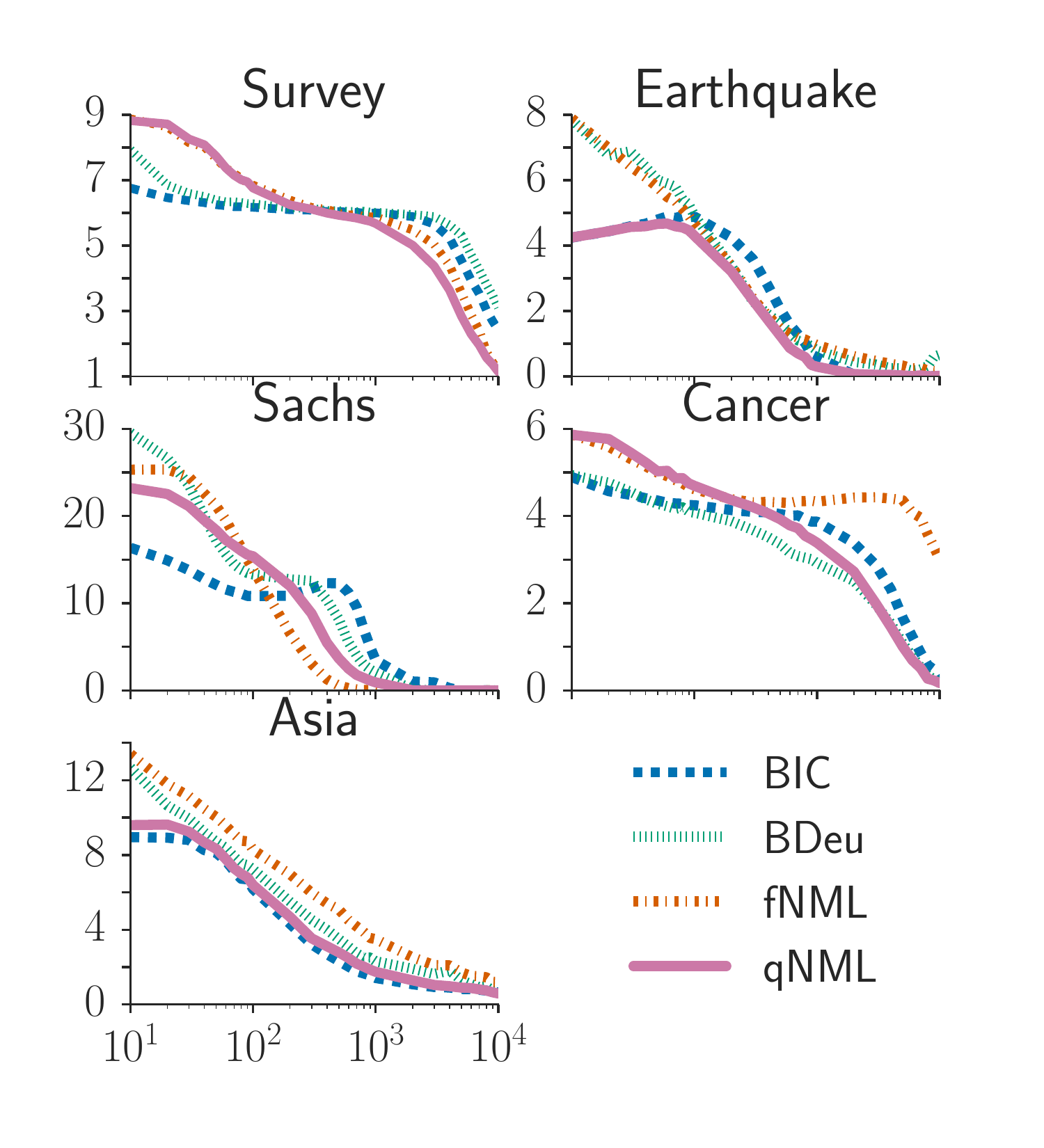}
\caption{Sample size versus SHD with data generated from real world DAGs.}
\label{fig:all_shd}
\end{figure}

Figure \ref{fig:shd_ranks} summarizes the SHD results in all networks by showing the average rank for each score. The ranking was done by giving the score with the lowest SHD rank $1$ and the worst one rank $4$. In case of ties, the methods with the same SHD were given the same rank. The shown results are averages computed from $5000$ values ($5$ networks, $1000$ repetitions). From this, we can see that qNML never has the worst average ranking, and it has the best ranking with sample sizes greater than $300$. This suggests that qNML is overall a safe choice in structure learning, especially with moderate and large sample sizes.


\begin{figure}[h]
\centering
\includegraphics[width=8cm,height=5cm]{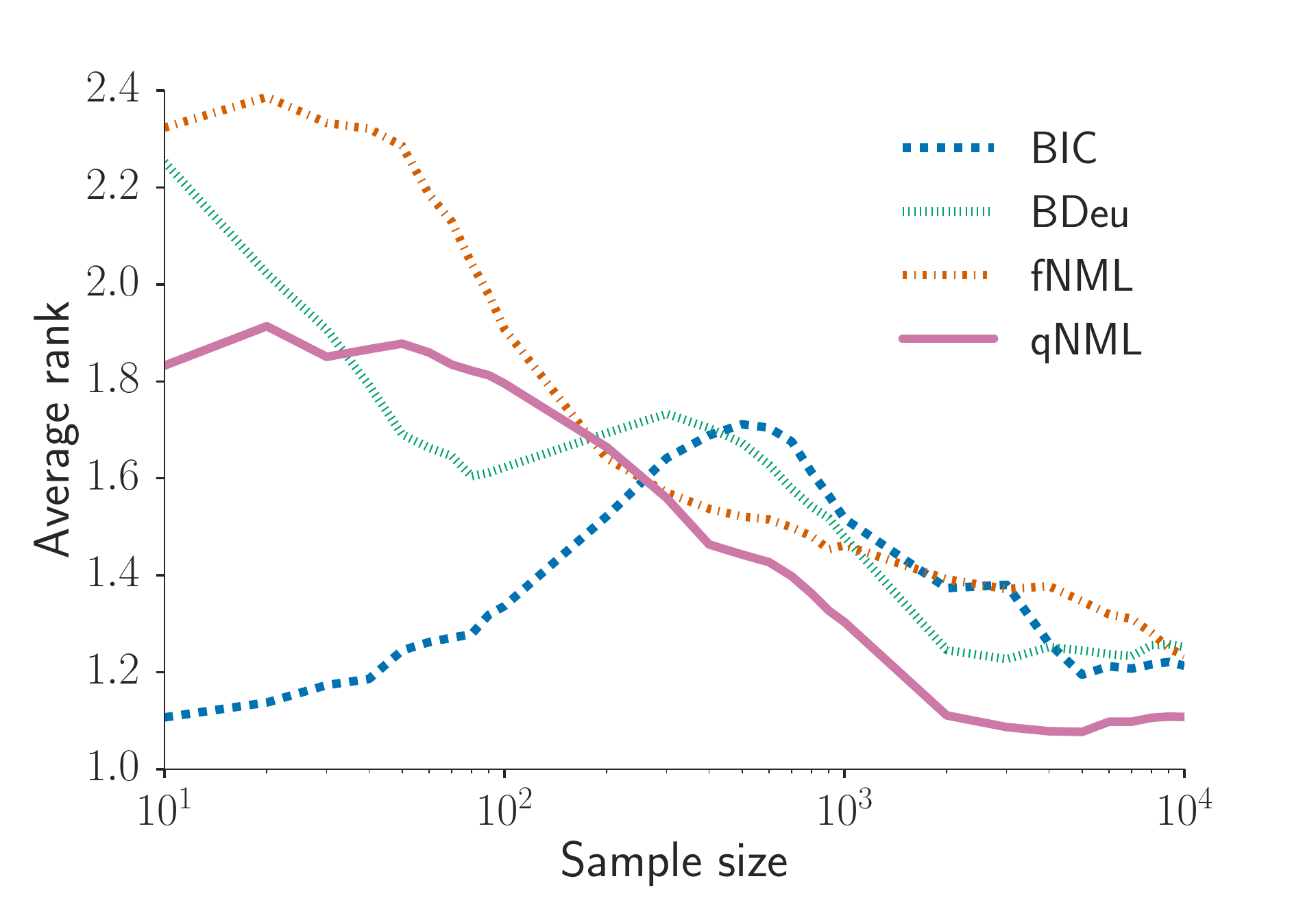}
\caption{Average ranks for the scoring functions in structure learning experiments.}
\label{fig:shd_ranks}
\end{figure}

%

\subsection{Prediction and Parsimony}

To empirically compare the model selection criteria, we took 20 UCI
data sets~\cite{Lichman:2013} and ran train and test experiments for
all of them. To better compare the performance over different sample
sizes, we picked different fractions ($10\%, 20\%, \ldots, 90\%$) of
the data sets as training data and the rest as the test data.
This was done for 1000 different permutations of each data set.
The training was conducted using the dynamic
programming based exact structure learning algorithm.

When predicting $P(d_{test}|D_{train},S,\theta)$
with structures $S$ learned by the BDeu score, we used the
Bayesian predictive parameter values (BPP) $\theta_{ijk} \propto
N_{ijk}+\frac{1}{r_iq_i}$.  In the spirit of keeping the scores
hyperparameter-free, for structures learned by the other model
selection criteria, we used the sequential predictive NML (sNML)
parametrization $\theta_{ijk}\propto e(N_{ijk})(N_{ijk}+1)$, where
$e(n)=(\frac{n+1}{n})^n$ as suggested in~\cite{Riss07b}.

\begin{table}
\caption{Average predictive performance rank over different sample sizes for different model selection criteria in 20 different data sets.}
\label{tbl:preds}
\begin{center}
\begin{tabular}{crrrrr}
       Data
    & N
    & \multicolumn{1}{p{0.7cm}}{\centering BDeu \\ BPP}
    & \multicolumn{1}{p{0.8cm}}{\centering BIC \\ sNML}
    & \multicolumn{1}{p{0.9cm}}{\centering fNML \\ sNML}
    & \multicolumn{1}{p{0.9cm}}{\centering qNML \\ sNML}\\
       \midrule
 PostOpe &     90 &              2.79 &     \textbf{1.20} &  \underline{3.06} &           2.94 \\
    Iris &    150 &  \underline{2.82} &              2.37 &     \textbf{2.27} &           2.54 \\
    Wine &    178 &  \underline{3.23} &     \textbf{1.88} &              2.67 &           2.22 \\
   Glass &    214 &  \underline{3.61} &              3.09 &     \textbf{1.42} &           1.88 \\
 Thyroid &    215 &              2.55 &  \underline{3.21} &     \textbf{1.80} &           2.44 \\
 HeartSt &    270 &  \underline{3.12} &     \textbf{1.39} &              3.12 &           2.37 \\
 BreastC &    286 &  \underline{3.09} &     \textbf{1.41} &              2.97 &           2.53 \\
 HeartHu &    294 &  \underline{3.18} &     \textbf{1.66} &              2.90 &           2.27 \\
 HeartCl &    303 &  \underline{3.46} &     \textbf{1.38} &              2.99 &           2.17 \\
   Ecoli &    336 &              3.20 &  \underline{3.53} &     \textbf{1.24} &           2.04 \\
   Liver &    345 &  \underline{3.17} &              2.39 &              2.69 &  \textbf{1.75} \\
 Balance &    625 &  \underline{3.35} &              1.91 &     \textbf{1.59} &           3.16 \\
 BcWisco &    699 &  \underline{3.06} &              2.03 &              2.89 &  \textbf{2.02} \\
 Diabete &    768 &  \underline{2.91} &              2.70 &              2.68 &  \textbf{1.71} \\
 TicTacT &    958 &  \underline{3.44} &              2.71 &     \textbf{1.31} &           2.53 \\
   Yeast &   1484 &              2.60 &  \underline{3.76} &     \textbf{1.55} &           2.10 \\
 Abalone &   4177 &              2.60 &  \underline{3.64} &     \textbf{1.04} &           2.72 \\
 PageBlo &   5473 &              2.24 &  \underline{3.61} &     \textbf{1.31} &           2.83 \\
   Adult &  32561 &              3.23 &  \underline{3.77} &     \textbf{1.00} &           2.00 \\
 Shuttle &  58000 &     \textbf{1.44} &  \underline{3.78} &              1.56 &           3.22 \\
\end{tabular}
\end{center}
\end{table}

For each train/test sample, we ranked the predictive performance of
the models learned by the four different scores (rank
1 being the best and 4 the worst). Table~\ref{tbl:preds} features the
average rank for different data sets, the average being taken over
1000 different train/test samples for each 9 sample sizes.
BIC's bias for simplicity makes it often win (written bold in
the table) with small sample sizes, but it performs worst
(underlined) for the larger sample sizes (for the same reason), while
fNML seems to be good for large sample sizes. The striking feature
about the qNML is its robustness.  It is usually between BIC and fNML
for all the sample sizes making it a ``safe choice''. This can be
quantified if we further average the columns of
Table~\ref{tbl:preds}, yielding the average ranks of $2.95, 2.57,
2.10$, and $2.37$, with standard deviations $0.49, 0.90, 0.76$, and
$0.43$.  While fNML achieves on average the best rank, the runner-up
qNML has the lowest standard deviation.

Figure~\ref{fig:bcnpmean} shows how fNML still sometimes behaves
strangely in terms of model complexity as measured by the number of
parameters in the model. qNML, instead, appears to yield more
parsimonious models. To study the concern of fNML producing too
complex models for small sample sizes, we studied the number of
parameters in models produced by different scores when using
10\% of each data set for structure learning.

Looking at the number of parameters for the same 20 data sets again
features BIC's preference for simple models
(Table~\ref{tbl:nofparams}).  qNML usually (19/20) yields more
parsimonious models than fNML that selects the most complex model for
7 out of 20 data sets.

The graphs for different sample sizes for both predictive accuracy and
the number of parameters can be found in Appendix E in the Supplementary Material.

\begin{table}
  \caption{Average number of parameters in models 
    for different scores in 20 different data sets.}
\label{tbl:nofparams}
\begin{center}
\begin{tabular}{crrrrr}
       Data
    & N
    & \multicolumn{1}{p{0.7cm}}{\centering BDeu\\ }
    & \multicolumn{1}{p{0.8cm}}{\centering BIC\\ }
    & \multicolumn{1}{p{0.9cm}}{\centering fNML\\}
    & \multicolumn{1}{p{0.9cm}}{\centering qNML\\}\\
\midrule
    Iris &    15 &     \underline{37} &   \textbf{23} &                33 &               29 \\
 PostOpe &    18 &   \underline{1217} &   \textbf{19} &               397 &              146 \\
   Ecoli &    34 &    \underline{182} &   \textbf{31} &               162 &               77 \\
   Liver &    35 &                 45 &   \textbf{15} &    \underline{61} &               24 \\
    Wine &    36 &  \underline{16521} &   \textbf{70} &               807 &              205 \\
   Glass &    44 &   \underline{1677} &   \textbf{48} &               506 &               97 \\
 Thyroid &    44 &                 40 &   \textbf{23} &    \underline{66} &               28 \\
 HeartSt &    54 &  \underline{16861} &   \textbf{44} &              1110 &              256 \\
 BreastC &    58 &  \underline{25797} &   \textbf{49} &              3767 &              844 \\
 HeartHu &    60 &   \underline{1634} &   \textbf{43} &               792 &               90 \\
 HeartCl &    62 &  \underline{34381} &   \textbf{47} &              1433 &              404 \\
 BcWisco &    70 &   \underline{4630} &   \textbf{42} &               603 &               89 \\
 Diabete &    77 &                 39 &   \textbf{22} &   \underline{216} &               34 \\
 TicTacT &    96 &  \underline{13701} &   \textbf{25} &              1969 &              767 \\
 Balance &   126 &        \textbf{20} &            24 &                49 &  \underline{611} \\
   Yeast &   149 &                 71 &   \textbf{31} &   \underline{265} &               75 \\
 Abalone &   418 &                 91 &   \textbf{46} &   \underline{150} &               63 \\
 PageBlo &   548 &    \underline{703} &   \textbf{45} &               380 &               56 \\
 Shuttle &  5800 &                535 &   \textbf{99} &   \underline{717} &              130 \\
   Adult &  6513 &                699 &  \textbf{479} &  \underline{1555} &              945 \\
\end{tabular}
\end{center}
\end{table}

\section{CONCLUSION}

We have presented qNML, a new model selection criterion for learning
structures of Bayesian networks.  While being competitive in
predictive terms, it often yields significantly simpler models than
other common model selection criteria other than BIC that has a very strong
bias for simplicity. The computational cost of qNML
equals the cost of the current state-of-the-art criteria. The
criterion also gives equal scores for models that encode the same
independence hypotheses about the joint probability distribution.
qNML also coincides with the NML criterion for many models.
In our experiments, the qNML criterion appears as a safe choice for a
model selection criterion that balances
parsimony, predictive capability and the ability to quickly converge to the
generating model.

\subsubsection*{Acknowledgements}

J.L., E.J. and T.R. were supported by the Academy of Finland (COIN CoE and Project {\sc TensorML}). J.L. was supported by the DoCS doctoral programme of the University of Helsinki.


\bibliographystyle{apalike}
\bibliography{cosco}

\onecolumn

\begin{appendices}
\section{Score equivalency proof}

Chickering showed that any equivalent structure can be reached from
another by a series of arc reversal operations without leaving the
equivalence class~\cite{Chick95}. This means that in order to show
that equivalent structures leads to the same qNML score, we only need
to prove that this is the case for equivalent structures that differ
by a single arc reversal. 

The network without any arcs is the sole member of its equivalence
class. All the other networks have at least one arc. We will first
present some lemmas that characterize the relations of the parent sets
before and after the arc reversal. Let $G$ be a network structure
and $G'$ an equivalent network structure after the arc from $A$
to $B$ has been reversed to point from $B$ to $A$. We
will denote the parent sets of $A$ and $B$ in $G$ by $G_{A}$
and $G_{B}$, and the parent sets of $A$ and $B$ in $G'$ by $G'_{A}$
and $G'_{B}$. Note that not all the arc reversals lead to the equivalent
structures. By saying that $G$ and $G'$ are equivalent we imply
that our arc reversal does not destroy existing V-structures or create
new ones. The crucial observation is that if reversing the arc that
goes from $A$ to $B$ does not create or destroy V-structures, it
must be that the parents of $B$ other than $A$ are exactly the same
as parents of $A$ (statement 3 below). 

\begin{lemma}\label{thm:sameparents}
\item $G_{B}=G_{A}\cup\{A\}$
\end{lemma}

\begin{proof}
We show the direction $G_{A}\cup\{A\}\subset G_{B}$ 
by contradiction. First of all, a parent of $A$ cannot be a child
of $B$ or otherwise there would be a loop in $G.$ If $A$ had a
parent $Z$ that is not a parent of $B$, then $B$ and $Z$ were
not adjacent and the reversal of the arc would create a new V-structure
$(B,A,Z)$. Since this was forbidden by the equivalence of $G$ and
$G'$, $Z$ must also be a parent of $B$. 

Also, $B$ cannot have parents other than $A$ that are not also parents
of $A$, i.e., $G_{B}\subset G_{A}\cup{A}$ . If there were such
a parent $Z$ it should be anyway adjacent to $A$, otherwise the
V-structure $(A,B,Z)$ would be destroyed in the reversal. Since the
$Z$ was not a parent of $A$, the only possibility for adjacency
is that $Z$ were a child of $A$. However, in this case arc reversal
would lead to a loop $B,A,Z,B$.
\end{proof}

We will next list some simple consequences of the equivalence preserving
arc reversal:

\begin{lemma}
\label{thm:reveqs}
\hangindent\leftmargini
\hspace{1mm}
\begin{enumerate}
\item $G'_{B}=G_{B}\setminus\{A\}$
\item $G'_{A}=G_{A}\cup\{B\}$ 
\item $G'_{B}=G_{A}$
\item $\{B\}\cup G_{B}=\{B\}\cup G_{A}\cup\{A\}=G'_{A}\cup\{A\}$
\item $G'_{A}=G_{A}\cup\{B\}=G'_{B}\cup\{B\}$
\end{enumerate}
\end{lemma}

\begin{proof}

The statements 1 and 2 are trivial since they just state the fact
that an arc reversal removes $A$ from the parents of $B$ and adds
$B$ to the parents of $A$. 

The
statement 3 follows directly by using the lemma \ref{thm:sameparents}
to the statement 1. Combining these equalities we can generate more of them 
such as statements 4 and 5. 
\end{proof}

Let us take structures $G$ and $G'$. Since we can assume the data to be fixed we lighten
up the notations and write 
$P^1_{NML}(i,G_i) := P^1_{NML}(D[\cdot,(i,G_i)];G)$ and
$P^1_{NML}(G_i) := P^1_{NML}(D[\cdot,G_i];G)$.

\begin{theorem}
$  s^{qNML}(D;G)=s^{qNML}(D;G').$
\end{theorem}

\begin{proof}

  Now the scores for structures can be decomposed as the
  $s^{qNML}(D;G)=\sum_{i=1}^{n}s_i^{qNML}(D;G)$ and 
  $s^{qNML}(D;G')=\sum_{i=1}^{n}s_i^{qNML}(D;G')$.

Since only the terms corresponding to the variables $A$ and $B$
in these sums are different, it is enough to show that

$$
s_A^{qNML}(D;G)+s_B^{qNML}(D;G) = s_A^{qNML}(D;G')+s_B^{qNML}(D;G')
$$
Now 

\begin{align*}
s_A^{qNML}(D;G)+s_B^{qNML}(D;G)& =\log\frac{P^1_{NML}(A,G_{A})}{P^1_{NML}(G_{A})}\frac{P^1_{NML}(B,G_{B})}{P^1_{NML}(G_{B})}\\
 & =\log 1\cdot\frac{P^1_{NML}(B,G_{B})}{P^1_{NML}(G_{A})}\\
 & =\log \frac{P^1_{NML}(B,G'_{B})}{P^1_{NML}(G'_{A})}\frac{P^1_{NML}(A,G'_{A})}{P^1_{NML}(G'_{B})}\\
 & =s_A^{qNML}(D;G')+s_B^{qNML}(D;G').
\end{align*}
The second equation follows from the lemma \ref{thm:sameparents}, and the third from
the statements 5 and 4 of the lemma \ref{thm:reveqs}.
\end{proof}
\newpage
\section{Regularity proof}

\subsection{Preliminaries}
We start by recalling the definition of regularity \citep{Suzuki2017}:

\begin{definition}
Assume $H_N(X \mid Y') \leq H_N(X \mid Y)$, where $Y' \subset Y.$ We say that the scoring function $Q_N(\cdot \mid \cdot)$ is regular if $Q_N(X \mid Y') \geq Q_N(X \mid Y)$.
\end{definition}
In the definition, $N$ denotes the sample size, $X$ is some random variable, $Y$ denotes the proposed parent set for $X$, and $H_N(X \mid Y)$ refers to the empirical conditional entropy based on $N$ samples of variables $X$ and $Y$.

Let $X$ be a categorical random variable with $r$ possible values. Let $U$ denote a possible parent set with $q$ different combinations of values for the variables, and $V$ a set with $m$ different configurations. Assume that we have observed $N$ samples of $(X,U,V)$ (denoted by $x_N,u_N$ and $v_N$) and $H_N(X \mid U) \leq H_N(X \mid U \cup V)$ holds.

Recall the definition of the qNML score:
\begin{align*}
Q^{qnml}_N(X\mid U) &= \log P(x_N \mid \hat{\theta}_{X\mid U} \ ) - \Le reg(N,rq) - reg(N,q) \Ri \\
&= \log P(x_N \mid \hat{\theta}_{X\mid U} \ ) - \log \frac{C(N,rq)}{C(N,q)},
\end{align*}where $C(N,r)$ is the normalizing constant the of the NML distribution for a categorical variable with $r$ possible values and sample size $N$ and $ \hat{\theta}_{X\mid U}$ denotes the maximum likelihood parameters of the conditional distribution of $X$ given $U$ which are computed from the data $(x_N,u_N)$.

In order to prove the regularity, we need the following three lemmas:
\begin{lemma}\label{Cpolynom} We can write $C(N,k)$ as a polynomial of $k$, formally
$$
C(N,k) = \sum_{j=1}^N a_j k^j,
$$where $a_j > 0$. 
\end{lemma}
\begin{lemma}\label{MLterms}
Assume $H_N(X \mid Y') \leq H_N(X \mid Y)$, where $Y' \subset Y$. Now $\log P(x_N \mid \hat{\theta}_{X\mid Y} \ ) = \log P(x_N \mid \hat{\theta}_{X\mid Y'} \ )$.
\end{lemma}
\begin{lemma}\label{increasing}
Let $r \in \N, r \geq 2$.  The function $k \mapsto \frac{C(N,rk)}{C(N,k)}$ is increasing for every $k \geq 2$.
\end{lemma}
\noindent We present the proofs of these lemmas in Section \ref{lemmaproofs}.  

\subsection{The main proof}
\begin{theorem}
qNML score is regular.
\end{theorem}
\begin{proof}
We want to show that
$$
Q^{qnml}_N(X\mid U) \geq Q^{qnml}_N(X \mid U \cup V).
$$assuming $H_N(X \mid U) \leq H_N(X \mid U \cup V)$. Using the entropy assumption and Lemma \ref{MLterms} implies that the maximized likelihood terms are equal. In order to prove the claim, it suffices to study the penalty terms, and we want to show that
\begin{align*}
- \Le reg(N,rq) - reg(N,q) \Ri \ &\geq \  -\Le reg(N,rqm) - reg(N,qm) \Ri \\
\log \frac{C(N,rq)}{C(N,q)} \ &\leq \log \frac{C(N,rqm)}{C(N,qm)}.
\end{align*}This holds, since logarithm is an increasing function, and $q \leq qm$, so we can apply Lemma \ref{increasing} to conclude the proof.

\end{proof}

\subsection{Proofs of lemmas}\label{lemmaproofs}
\setcounter{lemma}{0}
\begin{lemma} $C(N,k)$ can be written as a polynomial of $k$, formally
$$
C(N,k) = \sum_{j=1}^N a_j k^j,
$$where $a_j > 0$. 
\end{lemma}
\begin{proof} \cite{cosco.itsl08} derive the following representation for the normalizing constant
\begin{align*}  
C(N,k) = \sum_{l=0}^{N-1}\frac{(N-1)^{\underline{l}}k^{\overline{l + 1}}}{N^{l+1} \ l!} ,
\end{align*}where $x^{\underline{l}}$ and $x^{\overline{l}}$ denote falling and rising factorials, respectively. 

We utilize the fact that the rising factorial can be represented as polynomial using unsigned Stirling numbers of the first kind (see \cite{Adamchik1997}, for instance)
\begin{align*}
C(N,k) &=  \sum_{l=0}^{N-1}\frac{(N-1)^{\underline{l}}k^{\overline{l + 1}}}{N^{l+1} \ l!} \\
&= \sum_{l=0}^{N-1} b_l \ k^{\overline{l+1}} \\
&= \sum_{l=0}^{N-1} b_l \Le \sum_{j=1}^{l+1}|s(l + 1,j)| \ k^j \Ri \\
&= \sum_{l=0}^{N-1} \Le \sum_{j=1}^{N}b_l \ |s(l + 1,j)| \ k^j \Ri \\
&= \sum_{j=1}^{N} \Le \sum_{l=0}^{N-1} b_l \ |s(l + 1,j)| \ k^j \Ri \\
&= \sum_{j=1}^{N} \Le \sum_{l=0}^{N-1} b_l \ |s(l + 1,j)| \Ri k^j \\
&= \sum_{j=1}^{N} a_j  k^j,
\end{align*}where $s(i,j)$ denotes the (signed) Stirling number of the first kind and
$$
a_j = \Le \sum_{l=0}^{N-1} \frac{(N-1)^{\underline{l}}}{N^{l+1}l!} \ |s(l + 1,j)| \Ri,
$$ $a_j > 0$ for all $j$. On the second row, we denoted $b_l = (N-1)^{\underline{l}}/(N^{l+1} l!)$. On the row 4, we used the property of Stirling numbers: $s(i,j) = 0$ for all $j > i$. 
\end{proof}
\begin{lemma}
Assume $H_N(X \mid Y') \leq H_N(X \mid Y)$, where $Y' \subset Y$. Now $\log P(x_N \mid \hat{\theta}_{X\mid Y} \ ) = \log P(x_N \mid \hat{\theta}_{X\mid Y'})$.
\end{lemma}
\begin{proof}
We can write the logarithm of the maximized likelihood, \\ $\log P(x_N \mid \hat{\theta}_{X\mid Y} \ )$, as follows \citep{KOLLER}
\begin{align*}
\log P(x_N \mid \hat{\theta}_{X\mid Y} \ ) & = -N \Le H_N(X) - I_N(X;Y) \Ri \\
&= -N \  H_N(X\mid Y),
\end{align*}where $I_N(\cdot;\cdot)$ is the empirical mutual information. This implies that the assumption
$$
H_N(X \mid Y') \leq H_N(X \mid Y)
$$ is equivalent to
$$
\log P(x_N \mid \hat{\theta}_{X\mid Y} \ ) \leq \log P(x_N \mid \hat{\theta}_{X\mid Y'} \ ) .
$$Actually we must have the equality holding in the above expression, since
$$
H_N(X \mid Y') < H_N(X \mid Y)
$$ would imply that
$$
I_N(X ; Z \mid Y ) < 0,
$$where $Z = Y \setminus Y'$, which is impossible.
\end{proof}
\begin{lemma}
Let $r \in \N, r \geq 2$.  The function $k \mapsto \frac{C(N,rk)}{C(N,k)}$ is increasing for every $k \geq 2$.
\end{lemma}

\begin{proof}
Lemma \ref{Cpolynom} lets us to write
\begin{equation}\label{C1}
C(N,k) = \sum_{j=1}^{N} a_j k^j
\end{equation}and, similarly,
\begin{equation}\label{C2}
C(N,rk) = \sum_{j=1}^{N} a_j  r^jk^j.
\end{equation}

In the following, we assume that $k$ is a real number. From (\ref{C1}) and (\ref{C2}), it is easy see that the derivative of the quotient, $d/dk (C(N,rk)/C(N,k))$, will be a ratio of two polynomials of $k$. Our goal is to show that the polynomial in the numerator has positive coefficients, which will guarantee the positivity of derivative for every $k > 0$, and thus imply that the original function is increasing (polynomial in the denominator is squared and non-zero for $k > 0$, so it can be ignored).  

Derivatives of (\ref{C1}) and (\ref{C2}) are obtained easily:
\begin{align*}
\frac{d}{dk}C(N,k) &= \sum_{j=1}^{N}ja_jk^{j-1} \\
&= \sum_{j=0}^{N-1}(j+1) a_{j+1}k^{j}
\end{align*}and
\begin{align*}
\frac{d}{dk}C(N,rk) &= \sum_{j=1}^{N}ja_jr^jk^{j-1} \\
&= \sum_{j=0}^{N-1}(j+1)a_{j+1}r^{j+1}k^{j}. 
\end{align*}Consider next the products found in the derivative of the quotient. We obtain
\begin{align*}
\Le\frac{d}{dk}C(N,rk)\Ri C(N,k) &= \Le \sum_{j=0}^{N-1}(j+1)a_{j+1}r^{j+1}k^{j} \Ri \Le \sum_{l=1}^{N} a_l  k^l \Ri \\
&= \sum_{i = 1}^{2N-1} \Le \sum_{j+l = i} (j+1)a_{j+1}r^{j+1}a_l  \Ri k^i
\end{align*}and
\begin{align*}
\Le\frac{d}{dk}C(N,k)\Ri C(N,rk) &= \Le \sum_{j=0}^{N-1}(j+1)a_{j+1}k^{j} \Ri \Le \sum_{l=1}^{N} a_l  r^lk^l  \Ri \\
&= \sum_{i = 1}^{2N-1} \Le \sum_{j+l = i}(j+1)a_{j+1}a_l  r^l   \Ri k^i.
\end{align*}
Subtracting these two expression yields
\begin{align*}
&\Le\frac{d}{dk}C(N,rk)\Ri C(N,k)-\Le\frac{d}{dk}C(N,k)\Ri C(N,rk) \\&= \sum_{i = 1}^{2N-1} \Le \sum_{j+l = i} (j+1)a_{j+1}r^{j+1}a_l  \Ri k^i - \sum_{i = 1}^{2N-1} \Le \sum_{j+l = i}(j+1)a_{j+1}a_l  r^l   \Ri k^i \\
&= \sum_{i = 1}^{2N-1} \Le \sum_{j+l = i}(j+1)a_{j+1}a_l  (r^{j+1}- r^l)   \Ri k^i
\end{align*}which is the polynomial in the numerator of the derivative of $C(N,rk)/C(N,k)$.
Next, we study the coefficient of $k^i$, if $i \leq N$
\begin{align*}
\sum_{j+l = i}(j+1)a_{j+1}a_l  (r^{j+1}- r^l) &= \sum_{l=1}^{i}(i-l+1)a_{i-l+1}a_l  (r^{i-l+1}- r^l) \\
& = \sum_{l=1}^{i}(i-l+1)c_l \\
&= \sum_{t =1}^{\lfloor i / 2 \rfloor}(i-t+1)c_t + (i-(i-t + 1)+1)c_{i-t+1} \\
&=  \sum_{t =1}^{\lfloor i / 2 \rfloor}(i-t+1)c_t + tc_{i-t+1} \\
&=  \sum_{t =1}^{\lfloor i / 2 \rfloor}(i-t+1)c_t - tc_{t} \\
&= \sum_{t =1}^{\lfloor i / 2 \rfloor}(i-2t+1)c_t.
\end{align*} On the first row, we re-wrote sum using only one running index. On the second row we denoted $c_l =a_{i-l+1}a_l  (r^{i-l+1}- r^l)$. On the third row, we re-arranged the sum so that we are summing over pairs of terms of the original sum: the first and the last term, the second and the second to last, and so on.  This resulting sum has $\lfloor i / 2 \rfloor$ terms. We have to use the floor-function since if $i$ is odd, there exists an index $l'$ in the original sum such that $r^{i-l'+1}-r^{l'} = 0$. On the fifth row, we make use of the identity $c_t = -c_{i-t+1}$ which is straightforward to verify. From the last row, we can observe that every term of the sum is positive since $i-2t+1$ and $r^{i-t+1}- r^t$ are both positive if $t \leq (i+1)/2$ which holds since $t$ ranges from $1$ to $\lfloor i / 2 \rfloor$.

Let us now consider the situation where $n < i \leq 2N-1$. We start with the special case where $i = 2N-1$. Then, we have only one term in the sum
\begin{align*}
\sum_{j+l = i}(j+1)a_{j+1}a_l  (r^{j+1}- r^l) &= \sum_{l=N}^{N}(2N-1-l+1)a_{2N-1-l+1}a_l  (r^{2N-1-l+1}- r^l) \\
&= Na_{N}a_N  (r^{N}- r^N) \\
& = 0.
\end{align*}Now, let $N < i < 2N-1$, we follow a similar procedure as before to manipulate the sum
\begin{align*}
\sum_{j+l = i}(j+1)a_{j+1}a_l  (r^{j+1}- r^l) &= \sum_{l=i-N + 1}^{N}(i-l+1)a_{i-l+1}a_l  (r^{i-l+1}- r^l) \\
& = \sum_{l=i-N + 1}^{N}(i-l+1)c_l \\
& = \sum_{t=1}^{2N-i}(N-t+1)c_{i-N+t} \\
& = \sum_{t=1}^{\lfloor N - i/2 \rfloor}(N-t+1)c_{i-N+t}\\
&+  (N-(2N-i-t+1)+1)c_{i-N+(2N-i-t+1)} \\
&= \sum_{t=1}^{\lfloor N - i/2 \rfloor}(N-t+1)c_{i-N+t} +  (i-N+t)c_{N-t+1} \\
&= \sum_{t=1}^{\lfloor N - i/2 \rfloor}(N-t+1)c_{i-N+t} -  (i-N+t)c_{i-N+t} \\
&= \sum_{t=1}^{\lfloor N - i/2 \rfloor}(N-t+1-(i-N+t))c_{i-N+t} \\
&= \sum_{t=1}^{\lfloor N - i/2 \rfloor}(2N-i-2t+1)c_{i-N+t}.
\end{align*}It is now easy to verify that $(2N-i-2t+1)$ and $c_{i-N+t}$ are positive if $t \leq N - (i-1)/2$ which holds since $t$ ranges from $1$ to $\lfloor N - i/2 \rfloor$. The floor function is again used when we sum over pairs of terms since if $i$ is odd there is a zero-term. 
Since all the coefficients are non-negative and the $k \geq 2$, the derivative is positive. This implies that the original function is increasing.
\end{proof}

\subsection{fNML is Regular}
In this section, we will show that fNML score is regular. Recall first the definition of fNML local score
\begin{equation}
Q^{fnml}_N(X  \vert \ U) = \log P(x_N \mid \hat{\theta}_{X\mid U} \ ) - \sum_{j = 1}^{r_U} reg(N_j,r_X),
\end{equation}where $r_X$ is the number of categories for $X$, $r_U$ denotes the number of possible configurations of variables in $U$, and $N_j$ is the number of times the $j^\text{th}$ configuration is observed in our samples $u_N$. Note that $reg(N_j,r_X) = 0$ for any configuration not observed in our sample. 
fNML criterion differs from qNML only by how the penalty term is defined. We can follow a highly similar strategy in order to prove the regularity for fNML. To this end, we need the following Lemma:
\begin{lemma}\label{regineq}
Let $N = N_1 + N_2$ and $r \in\N, r \geq 2$. Now,
$$
reg(N,r) \leq reg(N_1,r) + reg(N_2,r).
$$
\end{lemma}
\begin{proof}
We start with the definition of $C(N_1,r) = \exp(reg(N_1,r))$,
\begin{align*}
C(N_1,r)C(N_2,r) &= \sum_{x_{N_1}}P(x_{N_1} \vert \hat{\theta}(x_{N_1} ))\sum_{x_{N_2}}P(x_{N_2} \vert \hat{\theta}(x_{N_2} ))\\ & \geq
\sum_{x_{N_1}}P(x_{N_1} \vert \hat{\theta}(x_{N} ))\sum_{x_{N_2}}P(x_{N_2} \vert \hat{\theta}(x_{N} )) \\ &=
\sum_{x_{N}}P(x_{N} \vert \hat{\theta}(x_{N} )) \\ & = C(N,r).
\end{align*}Taking the logarithm on both sides yields our claim. On the second row, we used the definition of maximum likelihood parameters: the probability of the data vector $x_{N_1}$ is maximized under parameters $\hat{\theta}(x_{N_1})$, and if we use different parameters, $\hat{\theta}(x_{N})$, the probability can only go down or stay the same. On the third row, we used the i.i.d., assumption, which allows us to take the single sum over data vectors of length $N = N_1 + N_2$.
\end{proof}
Now we can proceed to the actual proof.
\begin{theorem}
fNML score is regular.
\end{theorem}
\begin{proof}
Using the entropy assumption and Lemma \ref{C2}, we can ignore the maximized likelihood terms and study the penalty terms. We want show that
\begin{equation}\label{fnmlPenalty}
\sum_{j = 1}^{r_{Y'}} reg(N'_j,r_X) \leq \sum_{l = 1}^{r_{Y}} reg(N_l,r_X),
\end{equation}where we used $r_{Y'}$ and $r_Y$ to denote the numbers of possible configurations for variables in $Y'$ and $Y$, respectively. Also, $N_j$ refers to the number times the $j$:th configuration of variables $Y$ is observed in our sample, and $N'_j$ denotes the same for $Y'$. 

Now, since $r_{Y'} \leq r_Y$, and we know that $\sum_j N'_j = \sum_l N_l = N$, we can just apply Lemma (\ref{regineq}) multiple times to conclude our proof. 
\end{proof}

\newpage
\section{qNML coincides with NML for many models}

\subsection{qNML Equals NML for Many Models}
The fNML criterion can be seen as a computationally feasible
approximation of the more desirable NML criterion.  However, the fNML
criterion equals the NML criterion only for the Bayesian network
structure with no arcs.  We will next show that the qNML criterion
equals the NML criterion for all the networks $G$ whose connected
components tournaments (i.e., complete directed acyclic subgraphs of
$G$).

\begin{theorem}
If $G$ consists of $C$ connected components $(G^1,\ldots,G^C)$ with
variable sets $(V^1,\ldots,V^C)$, then $\log P_{NML}(D;G) = s^{qNML}(D;G)$
for all data sets $D$.
\end{theorem}
\begin{proof}
We first show that the NML-criterion for a Bayesian network
decomposes by the connected components.

Because the maximum likelihood for the data $D$ decomposes,
we can write
\begin{eqnarray}
  \lefteqn{P_{NML}(D;G)} \nonumber \\
  && = \frac{P(D;\hat\theta(D),G)}
            {\sum_{D'_{V_1}}\ldots \sum_{D'_{V_C}}\prod_{c=1}^C P(D'_{V^c};\hat\theta(D'_{V^c}),G)} \nonumber \\
            && = \frac{\prod_{c=1}^C P(D_{V^c};\hat\theta(D_{V^c}),G)} 
            {\prod_{c=1}^C\sum_{D'_{V^c}}P(D'_{V^c};\hat\theta(D'_{V^c}),G)}.\nonumber \\
            && = \prod_{c=1}^CP_{NML}(D_{V^c};G).
\end{eqnarray}

Clearly, the qNML score also decomposes by the connected components,
so it remains to show that if the (sub)network $G$ is a tournament,
then for any data $D$, $s^{qNML}(D;G)=\log P_{NML}(D;G)$.  Due to the
score equivalence of the NML criterion and the qNML criterion, we may
pick a tournament $G$ such that the linear ordering defined by $G$
matches the ordering of the data columns, i.e., $i<j$ implies $G_i
\subset G_j$. Now from the definition~(8) of the qNML
criterion (in the main paper), we see that for the tournament $G$, the sum telescopes
leaving us with $s^{qNML}(D;G) = \log P^1_{NML}(D_G;G)$, thus
it is enough to show that $P^1_{NML}(D;G)=P_{NML}(D;G)$.  This
follows, since for any data vector $x$ in data $D$, we have
$P^1(x;\hat\theta(D),G) = P(x;\hat\theta(D),G)$, where $P^1$ denotes
the model that takes $n$-dimensional vectors to be values of the
single (collapsed) categorical variable.  Denoting prefixes of data
vector $x$ by $x^{:i}$, and the number of times such a prefix appears
on $N$ rows $[d_1,\ldots,d_N]$ of the $N\times n$ data matrix $D$ by
$N_D(x^{:i})$, (so that $N_D(x^{:0})=N)$, we have
\begin{eqnarray}
  P(D;\hat\theta(D),G) &=& \prod_{j=1}^{N}P(d_j;\hat\theta(D),G) \nonumber \\
  &=& \prod_{j=1}^{N} \prod_{i=1}^{n} \frac{N_D(d_j^{:i})}{N_D(d_j^{:i-1})} \nonumber\\
  &=& \prod_{j=1}^{N} \frac{N_D(d_j^{:n})}{N} \nonumber \\
  &=& P^1(D;\hat\theta^1(D),G).
\end{eqnarray}
Since both $P^1_{NML}$ and $P_{NML}$ are defined in terms of the
maximum likelihood probabilities, the equality above implies the equality
of these distributions.
\end{proof}

Equality established, we would like to still state the number $a(n)$
of different $n$-node networks whose connected components are
tournaments.  We start by generating all the $p(n)$ integer partitions
i.e. ways to partition $n$ labelled into parts with different
size-profiles.  For example, for $n=4$, we have $p(4)=5$ partition
profiles $[[4], [3, 1], [2, 2], [2, 1, 1], [1, 1, 1, 1]]$. Each of
these partition size profiles corresponds to many different networks,
apart from the last one that corresponds just to the empty network.
We count number of networks for one such partition size-profile (and
then later sum these counts up).  For any such partition profile
$(p_1,\ldots,p_k)$ we can count the ways we can assign the nodes to
different parts and then order each part. This leads to the product
$n\choose p_1$$p_1$!${n-p_1}\choose p_2$$p_2!$${n-p_1-p2}\choose
p_3$$p_3!$$\ldots$${n-\sum_{j=1}^{k-1}p_j}\choose p_k$$p_k!$. However,
the order of different parts of the same size does not matter, so for
all groups of parts having the same size, we have to divide the
product above by the factorial of the size of such group. Notice also
that the product above telescopes, leaving us a formula for OEIS
sequence A000262\footnote{https://oeis.org/A000262} as
described by Thomas Wieder:

\textit{With $p(n) =$ the number of integer partitions of $n$, $d(i)$ = the
number of different parts of the $i^\text{th}$ partition of $n$,
$m(i,j) =$ multiplicity of the $j^\text{th}$ part of the $i^\text{th}$
partition of $n$, one has:}
$$
a(n) = \sum_{i=1}^{p(n)} \frac{n!}{\prod_{j=1}^{d(i)} m(i,j)!}.
$$
For example, $a(4)=\frac{4!}{1!}+\frac{4!}{1!1!}+\frac{4!}{2!}+\frac{4!}{1!2!}+\frac{4!}{4!}=73$. In general this sequence grows rapidly; $1, 1, 3, 13, 73, 501, 4051, 37633, 394353, 4596553, \ldots$.
\newpage
\section{Bayesian Dirichlet quotient score}

\cite{Suzuki2017} independently suggested a Bayesian Dirichlet
quotient score BDq, in which the NML-distributions
$P^1_{NML}(D_{i,G_i};G)$ and $P^1_{NML}(D_{G_i};G)$ of the qNML are
replaced by the Bayesian marginal likelihoods
$P^1(D_{i,G_i};G,\alpha)$ and $P^1_{NML}(D_{G_i};G,\alpha)$ using
parameter prior $\theta_{ij} \sim
Dirichlet(\alpha,\ldots,\alpha)$. Suzuki further suggested using the
Jeffreys' prior with $\alpha=\frac{1}{2}$. While this is a
popular choice in statistics, it is also possible to argue for other
values of $\alpha$ like $\alpha=\frac{1}{2}-\frac{\ln 2}{2\ln
  N}$~\citep{watanabe15a} and $\alpha=\frac{1}{3}$~\citep{jaasaari18}.

Our initial experiments indicate that model selection by BDq is still
highly sensitive to the hyperparameter $\alpha$. On the next 5 pages we
present results for learning the posterior distribution of the network
structures for the four predictor variables (discretized to three
equal width bins) in the iris data (N=150). The structure prior is
assumed to be uniform. Since BDq is score equivalent, we only include
one representative from each equivalence class. The images also show
the most probable structure and the predictive distribution given by
the most probable structure and the predictive distribution when
marginalizing over all the structures.

One can see how the structure posterior is quite sensitive to $\alpha$
and the selected network is different for $\alpha$-values $0.2$, $0.3$
and $0.5$.

\includepdf[pages={1-5},angle=90]{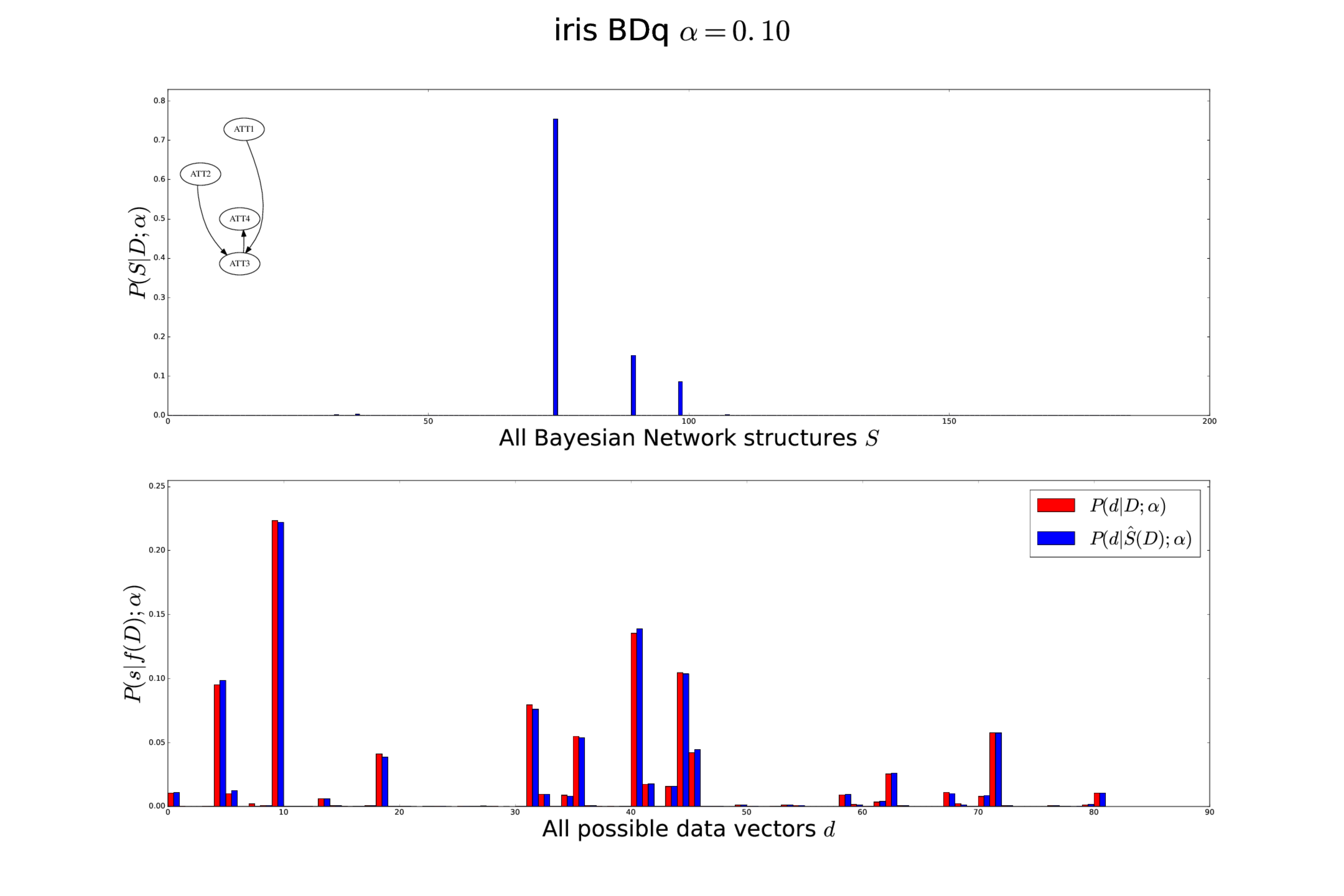}

\newpage
\section{Prediction and Parsimony}

We took 20 UCI data sets and created 1000 permutations of each data set.
We then took the first $x$\% of each permuted data set ($x$ in
${10,20,30,\ldots,90}$) as training data and used the remaining part as
test data. We used exact structure learning to learn the best
scoring model using BDeu ($\alpha=1$), BIC, fNML and qNML,
recording the number of parameters in the learnt
models and the average predictive log-probability of each test
vector. These numbers are collected for 1000 different data
permutations, the average and variance of which appear on the following
pages, one page per data set.

\includepdf[pages={1-20}, angle=90]{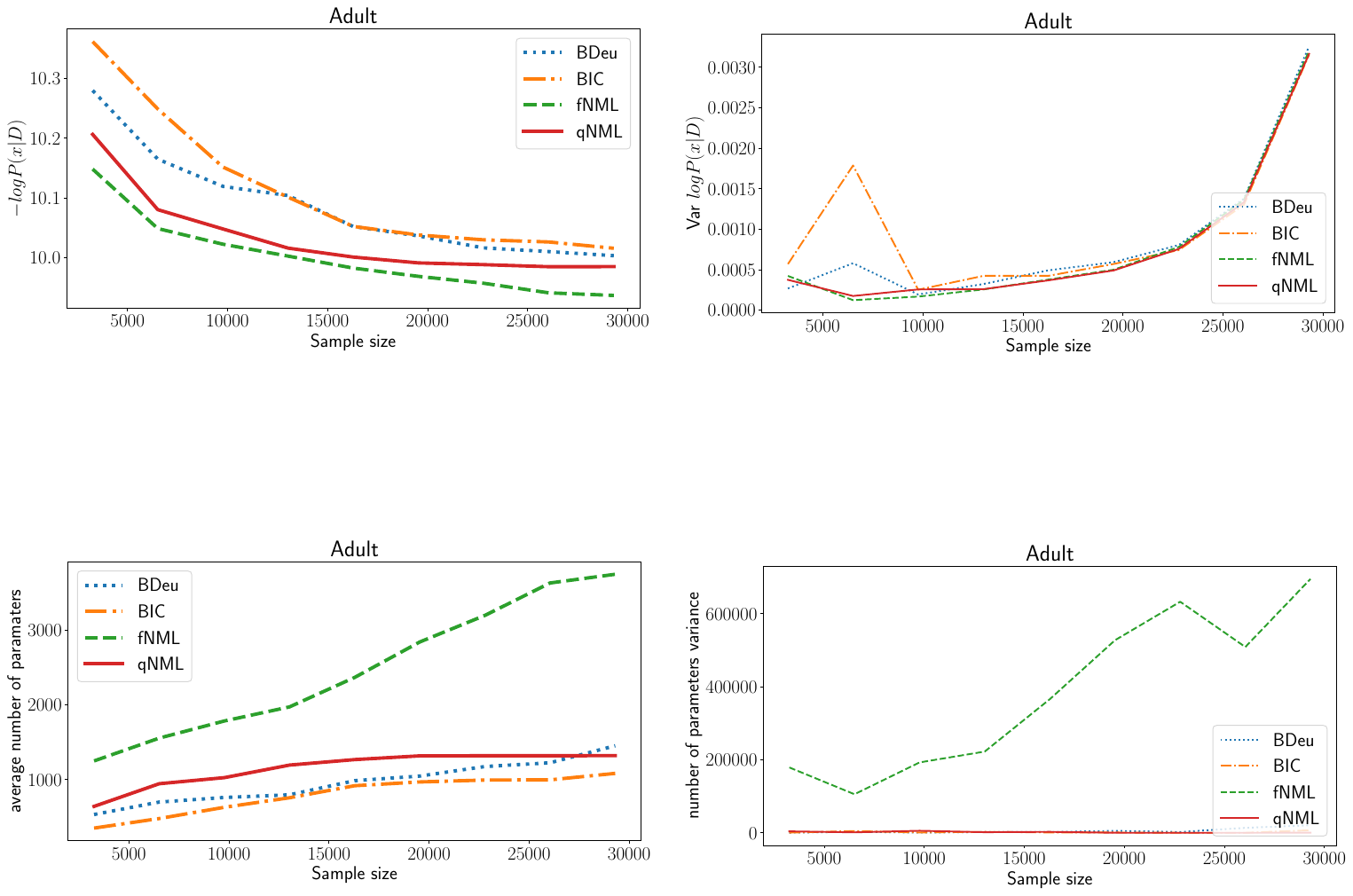}

\end{appendices}

\end{document}